\documentclass{article}

\usepackage{neurips_2026}

\usepackage[utf8]{inputenc} 
\usepackage[T1]{fontenc}    
\usepackage{hyperref}       
\usepackage{url}            
\usepackage{booktabs}       
\usepackage{amsfonts}       
\usepackage{nicefrac}       
\usepackage{microtype}      
\usepackage{xcolor}         

\usepackage{hyperref}
\usepackage{url}
\usepackage{booktabs}

\usepackage{placeins}
\usepackage{float}
\usepackage{amsmath}
\usepackage{amssymb}
\usepackage{mathtools}
\usepackage{amsthm}

\usepackage{bbm}

\usepackage{listings}
\usepackage{xcolor}

\lstdefinestyle{praxiscode}{
  basicstyle=\ttfamily\footnotesize,
  breaklines=true,
  breakatwhitespace=false,
  columns=fullflexible,
  keepspaces=true,
  showstringspaces=false,
  frame=single,
  framerule=0.3pt,
  rulecolor=\color{black!25},
  xleftmargin=0.0em,
  xrightmargin=0.0em,
  aboveskip=0.6em,
  belowskip=0.6em
}

\usepackage{algorithm}
\usepackage{algorithmic} 

\usepackage{siunitx}
\usepackage[capitalize,noabbrev]{cleveref}

\theoremstyle{plain}
\newtheorem{theorem}{Theorem}[section]

\newtheorem{lemma}[theorem]{Lemma}
\newtheorem{corollary}[theorem]{Corollary}
\theoremstyle{definition}
\newtheorem{definition}[theorem]{Definition}

\theoremstyle{remark}
\newtheorem{remark}[theorem]{Remark}

\usepackage{lineno}
\nolinenumbers

\definecolor{darkblue}{rgb}{0, 0, 0.5}
\hypersetup{colorlinks=true, citecolor=darkblue, linkcolor=darkblue, urlcolor=darkblue}

\title{PRAXIS: Learning Dynamics of Self-Improving Models with Symbolic Archives}

\author{Antiquus S.~Hippocampus, Natalia Cerebro \& Amelie P. Amygdale \thanks{ Use footnote for providing further information
about author (webpage, alternative address)---\emph{not} for acknowledging
funding agencies.  Funding acknowledgements go at the end of the paper.} \\
Department of Computer Science\\
Cranberry-Lemon University\\
Pittsburgh, PA 15213, USA \\
\texttt{\{hippo,brain,jen\}@cs.cranberry-lemon.edu} \\
\And
Ji Q. Ren \& Yevgeny LeNet \\
Department of Computational Neuroscience \\
University of the Witwatersrand \\
Joburg, South Africa \\
\texttt{\{robot,net\}@wits.ac.za} \\
\AND
Coauthor \\
Affiliation \\
Address \\
\texttt{email}
}

\begin{document}

\maketitle

\begin{abstract}
Self-improving learning systems adapt data selection, optimization, and auxiliary symbolic components, inducing nonstationary objectives outside standard learning assumptions. We introduce \textsc{PRAXIS}, a co-evolutionary framework that models generators, learners, and symbolic archives as interacting dynamical processes. We prove that KL-constrained generator updates and controlled archive-weight movement bound one-step objective drift, that archive updates suppress a program relative to any fixed comparator with a persistent cumulative utility advantage under sub-Gaussian noise, and that stochastic gradient descent achieves an average-stationarity guarantee whose degradation is governed by cumulative objective drift. Experiments across visual robustness, relational graph reasoning, and algorithmic graph reasoning exhibit generator stabilization, decreasing learner loss, and archive concentration consistent with these theoretical mechanisms.
\end{abstract}
\section{Introduction}
Modern learning systems increasingly adapt not only model parameters, but also the data, tasks, and auxiliary inductive signals that shape training. This pattern appears in self-training~\citep{wu2025selfai, behrouz2025nested}, curriculum learning~\citep{bengio2009curriculum, shi2025efficient, wang2023efficienttrain}, dataset distillation~\citep{li2024self, howard2024neurocomparatives, li2021self}, active learning~\citep{tamkin2022active, parvaneh2022active}, and self-evolving language models~\citep{huang2025r, yu2025guided, zhang2025path}. Yet these mechanisms are often specified externally or updated independently of the learner's internal dynamics: data distributions, supervision signals, task families, reward models, and verifiers are typically fixed or treated as exogenous. This limits our ability to reason about objective drift, stability, and adaptation of inductive structure, especially in settings such as visual perception and symbolic reasoning where reliable semantic oracles may be unavailable. We introduce \textsc{PRAXIS}, a learning-theoretic framework for self-improving learning as a closed-loop dynamical process coupling adaptive universe selection, gradient-based optimization, and symbolic archive regularization. Unlike co-adaptive pipelines that demonstrate empirical gains without a unified dynamics model~\citep{solko2024data,naik2024dolphin}, PRAXIS explicitly represents the interaction between the Generator, Learner, and symbolic Archive. This enables formal analysis of KL-controlled objective drift, learner tracking under nonstationarity, and selective suppression of non-competitive symbolic programs.

Our key contributions are:
i). a formal Generator--Learner--Archive model for self-improving learning under endogenous nonstationarity;
ii). finite-horizon guarantees for KL- and archive-controlled objective drift, average learner stationarity, and fixed-comparator archive suppression;
iii). empirical evidence across visual robustness, relational graph reasoning, and algorithmic graph reasoning showing the predicted stability and suppression dynamics.
\section{Preliminaries}
Let $\mathbb{R}$ denote the real numbers, $\Delta(\mathcal{I})$ the probability
simplex over a finite index set $\mathcal{I}$, and $\mathbb{E}[\cdot]$ expectation
with respect to the distribution specified by context.
Let $\mathcal{U}$ denote a universe space indexing a family of supervised
data-generating distributions.
Each universe $u \in \mathcal{U}$ induces a distribution
$D_u$ over labeled examples $(x,y)\in\mathcal{X}\times\mathcal{Y}$, differing
in properties such as data volume, bias, augmentation strength, or noise.
The simplex $\Delta(\mathcal{U})$ denotes distributions over universes, and
$\pi_0\in\Delta(\mathcal{U})$ denotes a fixed reference prior. 
For $q,\pi\in\Delta(\mathcal{U})$, the Kullback--Leibler (KL) divergence is defined as
\(
\mathrm{KL}(q\Vert\pi)
=
\int_{\mathcal{U}} \log\!\left(\frac{dq}{d\pi}\right)\,dq,
\)
whenever $q$ is absolutely continuous with respect to $\pi$, and infinite
otherwise.
A \emph{Learner} is a parametric predictor
\(
f_\theta:\mathcal{X}\rightarrow\Delta(\mathcal{Y}),\qquad \theta\in\Theta,
\)
mapping inputs to predictive label distributions.
The predictive output $z_\theta(x):=f_\theta(x)$ constitutes the sole interface
exposed to symbolic components; all symbolic interaction operates directly on
the Learner's predictive distribution.
A \emph{Generator} is a stochastic policy
\(
q_\phi\in\Delta(\mathcal{U}),\qquad \phi\in\Phi,
\)
which specifies a sampling distribution over universes and thereby determines
the mixture distribution from which training data are drawn.
The framework maintains a finite \emph{archive} of symbolic programs
$\mathcal{P}=\{p_1,\ldots,p_m\}$.
Each program
$p_j:\Delta(\mathcal{Y})\rightarrow\mathcal{Y}$
is a deterministic rule mapping the Learnerâ€™s predictive distribution to a hard
label.
Programs encode heuristic or structural biases and are not assumed to be
semantically correct.
Each program $p_j$ is associated with a score $s_j\in\mathbb{R}$, inducing a
Gibbs distribution over the archive:
\(
w_j=\frac{\exp(\beta s_j)}{\sum_{k=1}^m \exp(\beta s_k)},\qquad \beta>0,
\)
where $w\in\Delta(\mathcal{P})$ represents a soft inductive bias over symbolic
hypotheses.

\section{Problem Formulation}
\label{sec:problem_formulation}

Learning is defined over a universe space $\mathcal{U}$ indexing a family of
supervised data-generating distributions $\{D_u\}_{u\in\mathcal{U}}$.
For a Learner $f_\theta:\mathcal{X}\rightarrow\Delta(\mathcal{Y})$, the population
loss on universe $u$ is
\(
L(\theta,u)
=
\mathbb{E}_{(x,y)\sim D_u}
\big[
\ell\big(f_\theta(x),y\big)
\big],
\label{eq:base-loss}
\)
where $\ell$ is a bounded supervised loss.
Given a universe distribution $q\in\Delta(\mathcal{U})$, the induced mixture risk
is
\(
\mathcal{R}(\theta,q)
=
\mathbb{E}_{u\sim q} L(\theta,u).
\label{eq:universe-risk}
\)
The Generator adapts the universe distribution to emphasize learning-relevant
regions of $\mathcal{U}$.
Deviations from a reference prior $\pi_0\in\Delta(\mathcal{U})$ are regularized
via a KL penalty, yielding the variational objective
\begin{equation}
\min_{\theta\in\Theta}
\max_{q\in\Delta(\mathcal{U})}
\left\{
\mathbb{E}_{u\sim q} L(\theta,u)
-
\frac{1}{\eta}\,
\mathrm{KL}(q\Vert\pi_0)
\right\},
\label{eq:variational-minmax}
\end{equation}
where $\eta>0$ controls the adaptation rate.
In practice, the maximization over $q$ is implemented through KL--proximal
updates, ensuring smooth generator dynamics and bounded objective drift.
The symbolic archive $\mathcal{P}=\{p_j\}_{j=1}^m$ introduces auxiliary inductive
structure through a nonnegative consistency functional.
Each program produces a discrete symbolic prediction
\(
p_j(f_\theta(x))\in\mathcal Y
\)
from the learner's predictive distribution.  This hard prediction is not
differentiated through.  Let
\(
\tilde y_{j,\theta}(x)=\operatorname{sg}\!\left[p_j(f_\theta(x))\right]
\)
denote the resulting stop-gradient pseudo-label, where $\operatorname{sg}[\cdot]$
means that the argument is treated as a constant during backpropagation.  The
program then induces the clipped cross-entropy discrepancy
\begin{equation}
D_j(\theta;x)
=
\frac{1}{C_\Psi}
\min\!\left\{
-\log f_\theta(x)_{\tilde y_{j,\theta}(x)},
C_\Psi
\right\},
\qquad
0\le D_j(\theta;x)\le 1,
\label{eq:archive_discrepancy}
\end{equation}
where $C_\Psi>0$ is a clipping constant.  Gradients of $D_j$ flow only through
the learner distribution $f_\theta(x)$; the discrete program output
$\tilde y_{j,\theta}(x)$ is treated as fixed for the purpose of the learner
update.
For a labeled example $(x,y)$, the archive-induced penalty is
\begin{equation}
\Psi(f_\theta,\mathcal{P};x,y)
=
\sum_{j=1}^m
w_j\,
g_j(f_\theta;x,y)\,
D_j(\theta;x),
\label{eq:psi-def}
\end{equation}
where $w_j$ is the current archive weight and
$g_j(f_\theta;x,y)\in[0,1]$ is an optional intervention gate that modulates
whether program $j$ contributes on the example.  The gate is not differentiated;
it is used only to control the strength of the learner-facing consistency
regularizer.
This yields the augmented per-universe loss
\begin{equation}
L_{\mathrm{aug}}(\theta,u)
=
\mathbb{E}_{(x,y)\sim D_u}
\left[
\ell\big(f_\theta(x),y\big)
+
\lambda\,\Psi(f_\theta,\mathcal{P};x,y)
\right],
\label{eq:aug-loss}
\end{equation}
with $\lambda>0$ controlling the strength of symbolic regularization.
Combining adaptive universe selection with archive-augmented learning yields the
regularized min--max objective
\begin{equation}
\min_{\theta\in\Theta}
\max_{q\in\Delta(\mathcal{U})}
\left\{
\mathbb{E}_{u\sim q}
L_{\mathrm{aug}}(\theta,u)
-
\frac{1}{\eta}\,
\mathrm{KL}(q\Vert\pi_0)
\right\}.
\label{eq:full-problem}
\end{equation}
This objective is not solved to equilibrium; instead, it induces a coupled
closed-loop dynamics in which the Learner and Generator evolve jointly.
The Archive is not optimized directly, and influences learning only through the
time-varying weights $w_j$ appearing in $\Psi$.
\section{Solution Approach}
\label{sec:approach}

PRAXIS realizes the variational formulation of
Section~\ref{sec:problem_formulation} as a discrete-time stochastic process.
At iteration $t=0,1,2,\ldots$, the system maintains a coupled state
\(
(\theta_t,\, s_t,\, \phi_t),
\)
where $\theta_t\in\Theta$ are Learner parameters, $s_t=(s_{1,t},\ldots,s_{m,t})$
are Archive scores, and $\phi_t\in\Phi$ parameterizes a Generator distribution
$q_t:=\pi_{\phi_t}\in\Delta(\mathcal U)$.
The induced Archive weights are
$w_{j,t}=\exp(\beta s_{j,t})/\sum_k \exp(\beta s_{k,t})$.
Updates are performed online and are not solved to equilibrium.
Instead, each component is updated using stochastic first-order information
computed under the current state of the others.
All coupling between components occurs through the augmented loss
$L_{\mathrm{aug}}$ and the current weights $(q_t,w_t)$, yielding a closed-loop
learning process in which data selection, inductive structure, and model
parameters evolve jointly.
\subsection{Sampling and Estimation}
At iteration $t$, the Generator distribution $q_t$ is used to sample
$K$ universes
\(
u_t^{(1)},\ldots,u_t^{(K)} \sim q_t
\)
independently.
From each sampled universe $u_t^{(k)}$, an empirical dataset
$\mathcal D_t^{(k)}$ is drawn from $D_{u_t^{(k)}}$.
These samples define a Monte Carlo approximation to expectations with respect
to the mixture distribution induced by $q_t$.
Let
\(
\mathcal D_t := \bigcup_{k=1}^K \mathcal D_t^{(k)}
\)
denote the aggregated dataset at iteration $t$.
This sampling mechanism provides a unified source of signal for the Learner,
Archive, and Generator updates, ensuring that all components respond to the
same underlying universe mixture at each iteration.

\subsection{Learner Update}

At each iteration $t$, the Learner produces predictive distributions
$z_{\theta_t}(x)=f_{\theta_t}(x)\in\Delta(\mathcal Y)$.
Symbolic programs in the archive operate exclusively on these predictive distributions, and 
no additional latent representations are assumed.
This design enforces a strict interface between neural and symbolic components.
For any universe $u\in\mathcal U$, the augmented population loss is
\begin{equation}
L_{\mathrm{aug}}(\theta_t,u)
=
\mathbb E_{(x,y)\sim D_u}
\Big[
\ell(f_{\theta_t}(x),y)
+
\lambda\,\Psi(f_{\theta_t},\mathcal P_t;x,y)
\Big],
\label{eq:learner_aug_loss}
\end{equation}
where $\ell$ is a bounded supervised loss and
$\Psi$ is the Archive-induced consistency functional.
By construction of the consistency functional (Section~\ref{sec:problem_formulation}),
$0\le \Psi\le 1$, and therefore
\(
0\le L_{\mathrm{aug}}(\theta,u)\le L_{\max}+\lambda=:M
\qquad
\forall\,\theta,u.
\)
Given $q_t$, the induced mixture objective is
\(
F_t(\theta)
=
\mathbb E_{u\sim q_t} L_{\mathrm{aug}}(\theta,u).
\)
The Learner update is implemented via a stochastic descent step
\(
\theta_{t+1}
=
\theta_t
-
\alpha_t\, g_t,
\label{eq:learner_update}
\)
where $g_t$ is an unbiased estimator of $\nabla F_t(\theta_t)$ computed from
minibatches drawn from $\mathcal D_t$.
Gradients propagate only through $f_{\theta_t}(x)$.  Symbolic programs are
treated as fixed, non-differentiable mappings: each hard program output is
converted into a stop-gradient pseudo-label, and the learner receives gradient
only from the clipped discrepancy in Eq.~\eqref{eq:archive_discrepancy}.
Thus $\Psi$ is a learner-facing regularizer, with each program's contribution
scaled by its weight $w_{j,t}$ and optional gate $g_j$.

\subsection{Archive Update}

At each iteration $t$, the framework maintains a finite archive of symbolic
programs $\mathcal P=\{p_1,\ldots,p_m\}$.
Each program
$p_j:\Delta(\mathcal Y)\rightarrow\mathcal Y$
is a deterministic mapping from the Learner's predictive distribution to a hard
label.
Programs encode heuristic or structural inductive biases and are not assumed to
be semantically correct.
The archive may be augmented at discrete expansion events; between such events,
$\mathcal P$ is treated as fixed, and learning dynamics act only on the associated
scores and weights.
Let
\(
\hat y_t(x)=\arg\max_{y'\in\mathcal Y} f_{\theta_t}(x)_{y'}
\)
denote the learner's hard prediction.
Each program $p_j$ is assigned a scalar utility
\(
U_{j,t}
=
F_{j,t}
-
\gamma B_{j,t},
\)
where $\gamma>0$ controls the penalty assigned to counterproductive symbolic
influence.
The fitness term
\(
F_{j,t}
=
\frac{1}{|\mathcal D_t|}
\sum_{(x,y)\in\mathcal D_t}
\mathbf 1\!\left[
p_j\!\left(f_{\theta_t}(x)\right)=y
\right]
\)
measures empirical agreement with ground-truth labels under the current universe
mixture, while the blame term
\(
B_{j,t}
=
\frac{1}{|\mathcal D_t|}
\sum_{(x,y)\in\mathcal D_t}
\mathbf 1\!\left[
\hat y_t(x)=y
\;\wedge\;
p_j\!\left(f_{\theta_t}(x)\right)\neq y
\right]
\)
penalizes symbolic disagreement only on samples already classified correctly by
the Learner, thereby gating blame by learner reliability.
The fitness $F_{j,t}$ and blame $B_{j,t}$ are archive-level statistics.  They
are not differentiated through and are not the same object as the per-example
discrepancy $D_j$ used in $\Psi$.  The discrepancy $D_j$ affects the learner
update through the differentiable regularizer in Eq.~\eqref{eq:psi-def}, while
$F_{j,t}$ and $B_{j,t}$ affect future archive influence only through the score
and weight updates below.
The empirical utility used in the update is
$\widehat U_{j,t}=U_{j,t}+\xi_{j,t}$, where $\xi_{j,t}$ captures stochastic
estimation noise due to finite sampling.
Archive scores evolve according to
\(
s_{j,t+1}
=
s_{j,t}
+
\eta_p\,\widehat U_{j,t},
\)
where $\eta_p>0$ controls the rate at which empirical utility is accumulated over
time.
Scores are mapped to archive weights via a Gibbs distribution
\(
w_{j,t}
=
\frac{\exp(\beta s_{j,t})}{\sum_{k=1}^m \exp(\beta s_{k,t})},
\qquad \beta>0.
\)
This update induces a competitive selection process in which symbolic programs
accumulate influence in proportion to sustained empirical utility, while programs
that are persistently uninformative or contradictory are progressively
downweighted.

\subsection{Generator Update}

The Generator maintains a parametric distribution
$q_t=\pi_{\phi_t}\in\Delta(\mathcal U)$ over the universe space and adapts it in
response to the evolving Learner--Archive state.
At iteration $t$, universes
$u_t^{(1)},\ldots,u_t^{(K)}\sim q_t$
are sampled and used to update the Learner.
After the Learner update $\theta_{t+1}$, each sampled universe is assigned a
scalar reward
\(
R_t^{(k)}
=
\widehat L_{\mathrm{aug}}(\theta_{t+1},u_t^{(k)})
+
\epsilon_t^{(k)},
\)
where $\widehat L_{\mathrm{aug}}$ is the empirical augmented loss evaluated
post-update and $\epsilon_t^{(k)}$ is an exploration noise term.
Post-update evaluation biases the Generator toward universes that remain
informative under adaptation.

The Generator update minimizes the KL-regularized objective
\[
\mathcal J_t(\phi)
=
-
\frac{1}{K}
\sum_{k=1}^K
\big(R_t^{(k)}-b_t\big)
\log \pi_\phi(u_t^{(k)})
\nonumber\\
+
\frac{1}{\eta_1}\,
\mathrm{KL}(\pi_\phi\Vert\pi_{\phi_t})
+
\frac{1}{\eta_2}\,
\mathrm{KL}(\pi_\phi\Vert\pi_0),
\label{eq:generator_objective}
\]
where $b_t$ is a variance-reduction baseline and $\pi_0$ is a fixed reference
prior.
The first KL term enforces a trust-region constraint limiting step-to-step
distributional movement, while the second prevents collapse onto degenerate
subsets of $\mathcal U$.

Updates are accepted only if the resulting distribution
$q_{t+1}=\pi_{\phi_{t+1}}$ satisfies
\(
\mathrm{KL}(q_{t+1}\Vert q_t)\le \kappa_t,
\label{eq:generator_trust_region}
\)
for a prescribed radius $\kappa_t>0$.
Generator parameters are updated via
\(
\phi_{t+1}
=
\phi_t
-
\eta_\phi\nabla_\phi \mathcal J_t(\phi_t),
\)
which corresponds to a natural-gradient step on $\Delta(\mathcal U)$ under a
softmax parameterization. In expectation, these dynamics approximate ascent on the KL-regularized inner
maximization of Eq.~\eqref{eq:full-problem}.
Collectively, these updates implement a stochastic primal--dual tracking
procedure for the regularized min--max objective in
Eq.~\eqref{eq:full-problem}, with the Learner performing descent in $\theta$ and
the Generator performing regularized ascent in $q$.

\section{Theoretical Analysis of PRAXIS Dynamics}
\label{sec:theoretical_analysis}
This section analyzes the stability properties of the learning dynamics induced
by PRAXIS.
Rather than establishing convergence to a stationary optimum, the analysis
focuses on structural guarantees that align directly with the implemented
update rules.
Specifically, KL-regularized Generator updates are shown to control
nonstationarity of the Learner's population objective (Subsection \ref{subsec:generator-induced drift}), Archive weights are shown
to suppress persistently uninformative symbolic programs under stochastic
utility estimation (Subsection \ref{subsec:archive_suppression_under_utility_advantage}), and the Learner is shown to track the resulting
time-varying objective under standard stochastic gradient assumptions (Subsection \ref{sec:learner_tracking}).
While interpretations of the theory are provided here, the proofs are provided in the Appendix.

\subsection{Generator-induced objective drift}
\label{subsec:generator-induced drift}
\begin{definition}[Time--Varying Mixture Objective]
\label{def:time-mixture-objective}
At iteration $t$, let $q_t := \pi_{\phi_t}\in\Delta(\mathcal{U})$ denote the
Generator distribution and let $L_{\mathrm{aug},t}$ denote the augmented loss
formed using the current archive state $(\mathcal P_t,w_t)$.
The associated time--varying mixture objective is
\(
F_t(\theta)
\;:=\;
\mathbb{E}_{u\sim q_t}\,L_{\mathrm{aug},t}(\theta,u),
\qquad \theta\in\Theta,
\label{eq:def_Ft}
\)
where $L_{\mathrm{aug}}(\theta,u)$ is the augmented per--universe loss defined in
Section~\ref{sec:problem_formulation}.
\end{definition}

\begin{definition}[Uniform Objective Drift]
\label{def:uniform_drift}
The one--step pathwise uniform objective drift at iteration $t$ is defined as
\(
\delta_t
\;:=\;
\sup_{\theta\in\Theta}
\big|F_{t+1}(\theta)-F_t(\theta)\big|.
\label{eq:def_delta}
\)
\end{definition}

\begin{remark}
The PRAXIS Generator update enforces a KL trust region, ensuring that consecutive
universe distributions do not change arbitrarily between iterations.
This constraint is the primary mechanism used to control objective drift.    
\end{remark}

\begin{theorem}[Generator- and archive-controlled objective drift]
\label{thm:kl_drift}
Let $\{q_t\}_{t\ge 0}\subset\Delta(\mathcal U)$ be the Generator sequence and
$F_t(\theta)=\mathbb{E}_{u\sim q_t}L_{\mathrm{aug},t}(\theta,u)$ the induced mixture
objective. Define the archive-induced loss drift
\(
a_t:=\sup_{\theta\in\Theta,u\in\mathcal U}
|L_{\mathrm{aug},t+1}(\theta,u)-L_{\mathrm{aug},t}(\theta,u)|.
\)
If the augmented losses satisfy the uniform bound
\(
\sup_{\theta\in\Theta,u\in\mathcal U}
|L_{\mathrm{aug},t}(\theta,u)|
\;\le\;
M,
\)
then, for every $t\ge 0$,
\(
\delta_t
\;\le\;
M\|q_{t+1}-q_t\|_1+a_t
\;\le\;
M\sqrt{2\,\mathrm{KL}(q_{t+1}\Vert q_t)}+a_t,
\)
whenever $\mathrm{KL}(q_{t+1}\Vert q_t)<\infty$.
If programs and intervention gates are fixed across the step and only the archive
weights change, then $a_t\le \lambda\|w_{t+1}-w_t\|_1$. Moreover, if
$\mathrm{KL}(q_{t+1}\Vert q_t)\to0$ and $a_t\to0$, then $\delta_t\to0$.
\end{theorem}

\paragraph{Interpretation.}
Theorem~\ref{thm:kl_drift} separates the two sources of objective movement:
Generator movement in KL geometry and change in the archive-induced loss. Thus
Generator stabilization alone controls the full objective only when the archive
is frozen or its induced loss drift is controlled separately.

While $\delta_t$ bounds worst-case objective drift, learning proceeds along a
specific stochastic trajectory.
The next result relates this uniform notion to the drift actually experienced
along the Learner iterates.

\begin{definition}[Along-iterate objective drift (magnitude)]
\label{def:Delta_t}
Let $\{\theta_t\}_{t\ge 0}$ denote the (possibly stochastic) sequence of Learner
iterates and $\{F_t\}_{t\ge 0}$ the associated mixture objectives.
The one--step objective drift along the iterates is defined as
\(
\Delta_t
\;:=\;
\mathbb{E}\!\left[
\big|F_{t+1}(\theta_{t+1}) - F_t(\theta_{t+1})\big|
\right],
\label{eq:def_Delta}
\)
where the expectation is taken over all algorithmic randomness up to iteration
$t+1$.
\end{definition}

\begin{lemma}[Along-iterate drift is dominated by uniform drift]
\label{lem:Delta_le_delta}
For all $t\ge 0$,
\(
\Delta_t
\;\le\;
\mathbb E[\delta_t].
\)
\end{lemma}

Lemma~\ref{lem:Delta_le_delta} states that the expected objective change
experienced by the Learner along its realized trajectory is bounded by the
worst--case (uniform) objective drift.
This permits uniform drift bounds to be used as conservative surrogates in
tracking analyses.
Under stronger assumptions on Generator stabilization, the objective drift decays at
a quantifiable rate as shown below.

\begin{definition}[Exponential Generator stabilization]
\label{def:gen_stabilization}
The Generator is said to stabilize exponentially after $T_0$ if there exist
constants $c>0$, $\kappa>0$, and $T_0\in\mathbb{N}$ such that
\(
\|q_{t+1}-q_t\|_1
\;\le\;
c\,e^{-\kappa(t-T_0)},
\qquad \forall\, t\ge T_0.
\label{eq:def_gen_stab}
\)
\end{definition}

\begin{corollary}[Exponential component stabilization implies objective stabilization]
\label{cor:generator_drift}
Under the assumptions of Theorem~\ref{thm:kl_drift}, suppose the Generator
stabilizes exponentially and $a_t\le d e^{-\rho(t-T_0)}$ for constants
$d,\rho>0$. Then for all $t\ge T_0$,
\(
\delta_t
\;\le\;
Mc\,e^{-\kappa(t-T_0)}+d e^{-\rho(t-T_0)},
\)
and the along-iterate drift satisfies
\(
\Delta_t
\;\le\;
\mathbb E\!\left[Mc\,e^{-\kappa(t-T_0)}+d e^{-\rho(t-T_0)}\right].
\)
\end{corollary}

\textbf{Remark (stabilization versus concentration).}
In finite universe spaces, measures such as entropy or the maximum mass
$\max_u q_t(u)$ are often used to diagnose Generator concentration.
Such concentration is compatible with stabilization measured in the
$\ell_1$ distance $\|q_{t+1}-q_t\|_1$ (equivalently, total variation distance) or
in KL divergence: the distribution may become increasingly peaked while the
one-step change between successive Generator iterates decays.
For this reason, stabilization in KL geometry (or $\ell_1$ distance), rather than
concentration alone, is the appropriate condition governing the drift of the
time-varying mixture objective.

\subsection{Archive suppression under persistent utility advantage}
\label{subsec:archive_suppression_under_utility_advantage}

\begin{definition}[Population and empirical archive utilities]
\label{def:archive_utilities}
For each program $p_j$ at iteration $t$, its population utility is
\(
U_{j,t} := F_{j,t} - \gamma B_{j,t},
\)
where $F_{j,t}$ and $B_{j,t}$ denote symbolic fitness and blame, respectively, as
defined in Section~\ref{sec:approach}.
These are non-differentiated archive update statistics computed after evaluating
programs on a batch or probe set.  They should be distinguished from the
learner-facing discrepancy $D_j$ in Eq.~\eqref{eq:archive_discrepancy}, which
enters $\Psi$ and is differentiated only through the learner distribution.
The empirical utility used to update the archive is
\(
\widehat U_{j,t} = U_{j,t} + \xi_{j,t},
\)
where $\xi_{j,t}$ represents stochastic estimation noise arising from finite
sampling over universes and data.
\end{definition}

\begin{definition}[Sub-Gaussian \emph{pairwise} utility estimation noise]
\label{def:archive_noise_subg_pairwise}
Let $\{\mathcal F_t\}_{t\ge 0}$ denote the natural filtration generated by the
PRAXIS dynamics. The noise process $\{\xi_{j,t}\}$ satisfies, for constants
$\mu\ge 0$ and $\sigma>0$:
(i) Bounded conditional bias:
For all $j,t$,
\(
\big|\mathbb E[\xi_{j,t}\mid \mathcal F_{t-1}]\big|\le \mu
\qquad \text{a.s.}
\)
(ii) Pairwise sub-Gaussian increments:
For all ordered pairs $(i,j)\in[m]\times[m]$, all $t$, and all $\lambda\in\mathbb R$,
\(
\mathbb E\!\left[
\exp\!\Big(
\lambda\Big(
(\xi_{j,t}-\xi_{i,t})-\mathbb E[\xi_{j,t}-\xi_{i,t}\mid \mathcal F_{t-1}]
\Big)
\Big)
\;\middle|\; \mathcal F_{t-1}
\right]
\le
\exp\!\left(\tfrac{\lambda^2(2\sigma)^2}{2}\right)
\qquad \text{a.s.}
\)
\end{definition}
\begin{lemma}[Learner-coupled lower bound on utility gaps]
\label{lem:endogenous_margin}
Assume $0\le F_{k,t}\le 1$ for all $k,t$, and that blame is learner-gated in the
sense that blame is incurred only on samples correctly classified by the learner.
Let
\(
e_t := \mathbb P(\hat y_t(x)\neq y)
\)
denote the learner population error at iteration $t$.
Then, for all $k,t$,
\(
0 \le B_{k,t} \le 1-e_t.
\)
Consequently, for any programs $i,j$, the population utility gap satisfies
\(
U_{i,t}-U_{j,t}
\ge
(F_{i,t}-F_{j,t})-\gamma(1-e_t).
\)
Define the endogenous margin lower bound
\(
\varepsilon^{\mathrm{endo}}_{i,j,t}
:=
(F_{i,t}-F_{j,t})-\gamma(1-e_t).
\)
Then $U_{i,t}-U_{j,t}\ge \varepsilon^{\mathrm{endo}}_{i,j,t}$ for all $t$.
\end{lemma}
\paragraph{Interpretation.}
Lemma~\ref{lem:endogenous_margin} gives a conservative learner-dependent utility
bound. It does not by itself imply that learner improvement enlarges utility
gaps: such a conclusion requires directional control of the blame difference
$B_{i,t}-B_{j,t}$.

\begin{theorem}[High-probability fixed-comparator archive suppression]
\label{thm:archive_suppression_subg}
Let archive scores evolve as
\(
s_{j,t+1}=s_{j,t}+\eta_p\,\widehat U_{j,t},
\qquad
w_{j,t}=\frac{e^{\beta s_{j,t}}}{\sum_{k=1}^m e^{\beta s_{k,t}}},
\)
with $\beta>0$, $\eta_p>0$, and $\widehat U_{j,t}=U_{j,t}+\xi_{j,t}$.
Assume the noise process satisfies Definition~\ref{def:archive_noise_subg_pairwise}.
Fix a nonempty subset $\mathcal S\subseteq[m]$. For each $j\notin\mathcal S$,
fix a comparator $i(j)\in\mathcal S$ and define the pairwise utility gap
\(
\varepsilon_{j,t}:=U_{i(j),t}-U_{j,t}.
\)
Then for any horizon $T\ge T_0$ and confidence level $\delta\in(0,1)$, with
probability at least $1-\delta$, the following holds uniformly for all
$t\in\{T_0,\dots,T\}$ and all $j\notin\mathcal S$:
\(
\frac{w_{j,t}}{\sum_{i\in\mathcal S} w_{i,t}}
\;\le\;
\frac{w_{j,T_0}}{w_{i(j),T_0}}
\exp\!\Big(
-\beta\eta_p\sum_{\tau=T_0}^{t-1}\varepsilon_{j,\tau}
+
\beta\eta_p\,\Gamma_{\mathrm{subG}}(t,\delta)
\Big),
\)
where
\(
\Gamma_{\mathrm{subG}}(t,\delta)
=
2\mu(t-T_0)
+
2\sigma\sqrt{2(t-T_0)\log\!\Big(\tfrac{m^2(T-T_0+1)}{\delta}\Big)}.
\)
\end{theorem}
\textbf{Interpretation.}
Theorem~\ref{thm:archive_suppression_subg} shows that a symbolic program that is
persistently less useful than one fixed comparator experiences decay in relative weight,
up to stochastic fluctuations induced by finite sampling.
The result holds with high probability under mild assumptions on utility
estimation noise and does not require the learner to be optimal or converged.
Crucially, suppression depends on cumulative utility gaps rather than
instantaneous performance, ensuring robustness to transient noise and short-term
reversals.
As a consequence, the archive dynamics favor sustained empirical usefulness
over momentary agreement, yielding a stable pruning effect even in a
nonstationary learning environment.

\begin{corollary}[Exponential suppression under a persistent signed margin]
\label{cor:archive_exponential_suppression}
Under the conditions of Theorem~\ref{thm:archive_suppression_subg}, suppose there
exist constants $\varepsilon>0$ and $T_1\ge T_0$ such that, for each
$j\notin\mathcal S$,
\(
\sum_{\tau=T_1}^{t-1}
\varepsilon_{j,\tau}
\;\ge\;
\varepsilon (t-T_1)
\qquad \forall\, t\le T.
\)
Then, with probability at least $1-\delta$, the archive weights satisfy
\(
\log\frac{w_{j,t}}{\sum_{i\in\mathcal S} w_{i,t}}
\;\le\;
 C_j
-
\beta\eta_p\,\varepsilon (t-T_1)
+
\beta\eta_p\,\Gamma_{\mathrm{subG}}(t,\delta),
\qquad
j\notin\mathcal S,
\)
for a constant $C_j$ depending on the score state at $T_1$.
In particular, if $\mu=0$, the exponent grows linearly in $t$, yielding
exponential suppression up to sublinear stochastic fluctuations.
\end{corollary}

Corollary~\ref{cor:archive_exponential_suppression} strengthens
Theorem~\ref{thm:archive_suppression_subg} by identifying a simple sufficient
condition under which archive suppression becomes exponential.
If each non-elite program has a persistent signed utility gap relative to a
fixed comparator, then its relative weight decays exponentially fast, up to
sublinear stochastic fluctuations induced by finite-sample noise.
In this regime, archive selection behaves analogously to a replicator dynamic
with endogenous margins: once a program is consistently dominated, its
influence vanishes at a rate determined by the learning rate, inverse
temperature, and the magnitude of the persistent utility gap.
This formalizes how long-term symbolic competition induces effective archive
compression without requiring explicit pruning heuristics.

\subsection{Average learner stationarity under objective drift}
\label{sec:learner_tracking}

\begin{theorem}[Average SGD stationarity under objective drift]
\label{thm:sgd_tracking_drift}
Let $\{q_t\}_{t\ge 0}$ denote the Generator sequence and define the time--varying
mixture objective
\(
F_t(\theta)
:=
\mathbb{E}_{u\sim q_t}\,L_{\mathrm{aug},t}(\theta,u),
\)
where $L_{\mathrm{aug},t}$ is the current augmented loss. Let $\mathcal G_t$
denote the pre-gradient sigma-field, with $\theta_t$, $q_t$, $w_t$, and $F_t$
measurable with respect to $\mathcal G_t$.
Assume that each $F_t$ is bounded below on $\Theta$, i.e., there exists
$F_{\min}\in\mathbb{R}$ such that
\(
F_t(\theta)\ge F_{\min}
\)
for all $t\ge 0$ and $\theta\in\Theta$.
Let the Learner update be
\(
\theta_{t+1} = \theta_t - \eta\, g_t,
\)
Assume that for all $t\ge 0$ the stochastic gradient $g_t$ satisfies:
\begin{enumerate}
\item[(i)] \emph{Unbiasedness:}
\(
\mathbb{E}\!\left[g_t \mid \mathcal G_t\right]
=
\nabla F_t(\theta_t).
\)
\item[(ii)] \emph{Bounded conditional variance:}
There exists $\sigma^2>0$ such that
\(
\mathbb{E}\!\left[
\|g_t - \nabla F_t(\theta_t)\|^2
\,\middle|\,
\mathcal G_t
\right]
\le
\sigma^2.
\)
\end{enumerate}
Assume further that $\Theta=\mathbb R^d$ and each $F_t$ is differentiable and
$L$--smooth, i.e.,
\(
\|\nabla F_t(\theta)-\nabla F_t(\theta')\|
\le
L\|\theta-\theta'\|,
\qquad
\forall\,\theta,\theta'\in\Theta.
\)
Then, for any horizon $T\ge 1$ and stepsize $\eta\in(0,1/L]$,
\(
\frac{1}{T}\sum_{t=0}^{T-1}
\mathbb{E}\!\left[\|\nabla F_t(\theta_t)\|^2\right]
\;\le\;
\frac{2}{\eta T}\big(F_0(\theta_0)-F_{\min}\big)
\;+\;
\eta L\,\sigma^2
\;+\;
\frac{2}{\eta T}\sum_{t=0}^{T-1}\mathbb E[\delta_t],
\)
where the one--step uniform objective drift is defined as
\(
\delta_t
:=
\sup_{\theta\in\Theta}
\big|F_{t+1}(\theta)-F_t(\theta)\big|.
\)
\end{theorem}

\paragraph{Interpretation.}
Theorem~\ref{thm:sgd_tracking_drift} bounds the average stationarity violation
$\frac{1}{T}\sum_{t=0}^{T-1}\mathbb E\|\nabla F_t(\theta_t)\|^2$ of stochastic
gradient descent when optimizing a time--varying mixture objective.
The first two terms coincide with standard nonconvex SGD guarantees for a fixed
objective, capturing optimization error due to initialization and gradient
noise.
The additional term $\frac{2}{\eta T}\sum_{t=0}^{T-1}\mathbb E[\delta_t]$ explicitly quantifies
the effect of objective nonstationarity, measuring how rapidly the
underlying objective itself changes across iterations.
Consequently, when Generator and archive movement make
$\sum_{t=0}^{T-1}\mathbb E[\delta_t]$ sublinear in $T$, the drift contribution to
average stationarity vanishes. This does not by itself establish
parameter-space tracking of a moving minimizer.

\section{PRAXIS Experimental Setup}
\label{sec:experiments}

We evaluate the same Generator--Learner--Archive process across three domains: MIT Indoor visual robustness, GraphLog relational reasoning, and CLRS algorithmic graph reasoning.  In MIT Indoor, universes encode image training conditions defined by draw count, augmentation strength, corruption type, and corruption severity; held-out universes use disjoint corruption/severity settings.  In GraphLog, universes correspond to graph-reasoning worlds and target relation families in binary relational query answering.  In CLRS, universes encode graph size, graph family, density, and edge-weight regime for a shortest-path threshold task; the hard split tests transfer to larger weighted graphs.  Dataset and universe details are provided next.

All experiments instantiate the same PRAXIS abstraction: a finite universe set
$\mathcal U$, a KL-regularized Generator over universes, a Learner trained on
samples from selected universes, and a symbolic Archive whose program weights
evolve by utility-based Gibbs updates. The three domains differ in modality and
task semantics, but not in the closed-loop structure being tested.

\begin{table}[h]
\centering
\scriptsize
\caption{Dataset scale and universe splits used in the multidomain evaluation.}
\label{tab:dataset_sizes}
\begin{tabular}{llll}
\toprule
Domain & Task scale & Train universes & Held-out universes \\
\midrule
MIT Indoor & $1{,}200$ images & $300$ & $27$ \\
GraphLog & $200$ tasks / $1{,}800$ binary examples & $43$ & $24$ \\
CLRS & $4{,}080$ generated graph examples & $16$ & $18$ \\
\bottomrule
\end{tabular}
\end{table}

\paragraph{MIT Indoor visual robustness.}
The visual-domain experiment uses a controlled subset of the MIT Indoor Scenes
dataset. We select $12$ scene classes with $100$ images per class, yielding a
balanced dataset of $1{,}200$ images. The selected classes are auditorium,
bedroom, bookstore, classroom, computer room, corridor, grocery store, kitchen,
library, living room, office, and restaurant.

Universes are defined by stochastic training conditions rather than by disjoint
image pools. Each MIT Indoor training universe is a tuple of training draw count
$\{400,800,1200,1600,2000\}$, augmentation strength from $5$ evenly spaced
values in $[0,0.5]$, corruption type
\{\texttt{none}, \texttt{blur}, \texttt{jpeg}, \texttt{noise}\}, and corruption
severity $\{0.0,0.3,0.6\}$. This gives $300$ training universes. Held-out
universes use evaluation draw count $2400$, augmentation strengths
$\{0.5,0.75,0.9\}$, corruption types
\{\texttt{blur}, \texttt{jpeg}, \texttt{noise}\}, and severities
$\{0.5,0.75,0.85\}$, giving $27$ evaluation universes.

\paragraph{GraphLog relational reasoning.}
The GraphLog experiment is formulated as binary relational query answering.
Each task contains labeled graph structure and query-answer examples. Each
example asks whether a target relation is entailed for a queried node pair under
the sampled rule world. The experiment used a 
rule-balanced file containing $200$
tasks and $1{,}800$ binary examples where universes are indexed
by rule world and target relation family. The fixed split used for all methods
contains $43$ training universes and $24$ held-out universes.

\paragraph{CLRS graph reasoning.}
The CLRS-domain experiment is a CLRS-style algorithmic graph-reasoning task
rather than a benchmark number on the full CLRS suite. Each example is a graph
instance with graph size, graph family, edge density, and edge-weight regime
parameters. The binary target asks whether a shortest-path quantity crosses a
threshold. The hard split contains $16$ training universes and $18$ held-out
universes. We generate $120$ examples per universe, for $4{,}080$ examples in
total across train and held-out universes, and use $40$ probe examples per
training universe for generator and archive diagnostics.

Baselines are matched by role across domains: fixed-mixture empirical risk minimzation (ERM), adaptive universe selection without archive regularization, Group Domain Random Optimization (DRO)-style and bandit-style samplers. Performance in comparison to PRAXIS is reported in Appendix~\ref{app:component_ablations}.  These comparisons contextualize predictive performance, while the main empirical question is whether the predicted Generator, Learner, and Archive dynamics recur across domains.

\section{Results}
\label{sec:results}

We organize results around three research questions (RQ) aligned with the three theoretical mechanisms. They are presented as follows.

\textbf{RQ1: Does the Generator adapt while remaining stable?}
Figure~\ref{fig:multidomain_generator_normalized} compares the temporal shape of Generator entropy, maximum universe probability, and KL drift across domains using within-domain min--max normalization.  The normalized curves show the same qualitative pattern: entropy decreases, maximum universe probability increases, and the KL drift proxy remains bounded rather than diverging.  Because normalization hides absolute scale, The table in Figure ~\ref{fig:multidomain_generator_normalized} reports the corresponding absolute quantities.
The absolute values show domain-specific concentration strength.  MIT Indoor remains comparatively diffuse, with entropy decreasing from $4.68$ to $3.91$ and maximum universe probability increasing from $0.009$ to $0.073$.  GraphLog concentrates sharply, with entropy decreasing from $3.14$ to $0.38$ and maximum probability increasing from $0.043$ to $0.925$.  CLRS lies between these cases, with entropy decreasing from $2.77$ to $2.31$ and maximum probability increasing from $0.063$ to $0.198$.  Despite these differences, the maximum observed drift proxy remains small in all cases: $0.0128$ for MIT Indoor, $0.0173$ for GraphLog, and $0.0016$ for CLRS.  This supports the prediction of Theorem~\ref{thm:kl_drift}: the Generator can concentrate the training distribution while keeping one-step movement controlled.

\begin{figure}[!h]
\centering

\begin{minipage}[c]{0.64\textwidth}
    \centering
    \includegraphics[width=\linewidth]{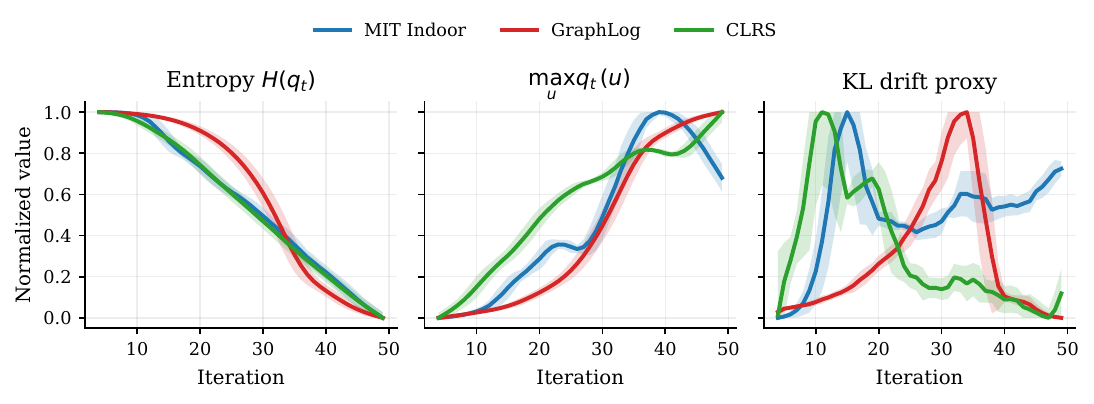}
\end{minipage}
\hfill
\begin{minipage}[c]{0.32\textwidth}
    \centering
    \scriptsize
    \setlength{\tabcolsep}{3pt}
    \renewcommand{\arraystretch}{0.95}
    \begin{tabular}{lccc}
    \toprule
    Domain & Ent. & Max $q$ & Drift \\
    \midrule
    MIT & $4.68{\to}3.91$ & $.009{\to}.073$ & .0128 \\
    GraphLog & $3.14{\to}.38$ & $.043{\to}.925$ & .0173 \\
    CLRS & $2.77{\to}2.31$ & $.063{\to}.198$ & .0016 \\
    \bottomrule
    \end{tabular}
\end{minipage}

\caption{Generator dynamics across domains. Curves show within-domain min--max normalized rolling means; bands show one rolling standard deviation.}
\label{fig:multidomain_generator_normalized}
\end{figure}

\textbf{RQ2: Does learner optimization remain stable under the induced nonstationarity?}
Figure~\ref{fig:multidomain_learner_stability} shows the learner-side response to the moving training distribution.  In all three domains, the rolling training cross-entropy decreases rather than destabilizing under adaptive sampling.  MIT Indoor decreases from $1.007$ to $0.080$, GraphLog from $0.702$ to $0.109$, and CLRS from $0.617$ to $0.421$ with a minimum of $0.361$.  These observations are consistent with stable optimization under a moving objective, but training loss is not a direct measurement of the gradient-norm quantity bounded in Theorem~\ref{thm:sgd_tracking_drift}.

\begin{figure}[!h]
\centering

\begin{minipage}[c]{0.32\textwidth}
  \centering
  \includegraphics[width=\linewidth]{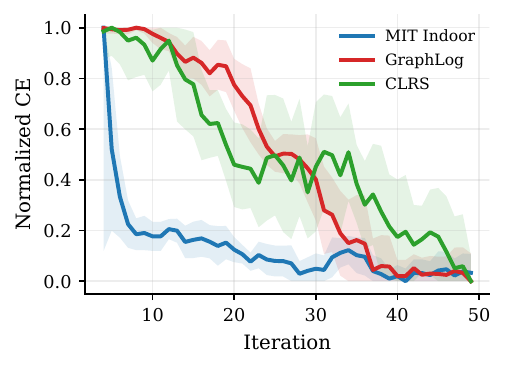}
\end{minipage}
\hfill
\begin{minipage}[c]{0.52\textwidth}
  \centering
  \scriptsize
  \begin{tabular}{lccc}
  \toprule
  Domain & First & Final & Min \\
  \midrule
  MIT & 1.007 & 0.080 & 0.068 \\
  GraphLog & 0.702 & 0.109 & 0.109 \\
  CLRS & 0.617 & 0.421 & 0.361 \\
  \bottomrule
  \end{tabular}
\end{minipage}

\caption{Normalized Learner stability across domains. Curves show rolling training cross-entropy min--max normalized within each domain; shaded bands show one rolling standard deviation after the same normalization.}
\label{fig:multidomain_learner_stability}
\end{figure}

\textbf{RQ3: Does the symbolic archive concentrate and suppress non-elite programs?}
Figure~\ref{fig:multidomain_archive_weights} shows archive weight evolution.  MIT Indoor exhibits moderate concentration, with maximum archive weight increasing from $0.120$ to $0.181$.  GraphLog and CLRS show much stronger concentration: GraphLog increases from $0.141$ to $0.987$, while CLRS increases from $0.164$ to $0.964$.  Archive entropy decreases in all three domains, indicating that the archive does not remain a diffuse pool of interchangeable programs.
Figure~\ref{fig:multidomain_suppression} reports the original envelope view: worst non-elite, median non-elite, and most-suppressed non-elite ratios relative to the final elite set with $k=2$.  The median non-elite ratio decreases from $-0.95$ to $-26.19$ on MIT Indoor, from $-0.49$ to $-8.03$ on GraphLog, and from $-1.04$ to $-18.29$ on CLRS.  This demonstrates empirical concentration. The fixed-comparator theorem predicts such decay only when a named comparator has a persistent signed cumulative utility advantage; verifying that assumption requires logging the corresponding pairwise utility gaps.

\begin{figure*}[!h]
  \centering
  \includegraphics[width=0.65\textwidth]{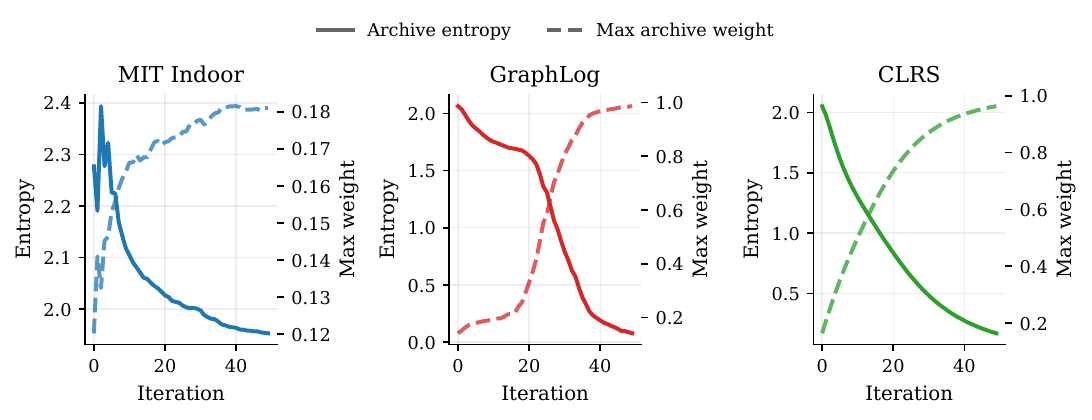}
  \caption{Archive weight evolution across domains.  Solid curves show entropy}
  \label{fig:multidomain_archive_weights}
\end{figure*}

\begin{figure*}[!h]
  \centering
  \includegraphics[width=0.65\textwidth]{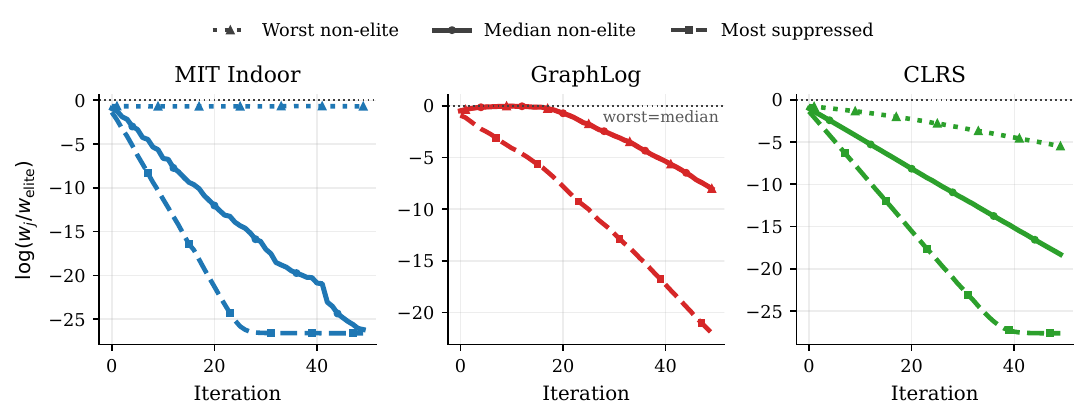}
  \caption{Archive suppression envelopes for elite size $k=2$.}
  \label{fig:multidomain_suppression}
\end{figure*}

\section{Conclusion and Future Work}

We introduced \textsc{PRAXIS}, a framework for modeling self-improving learning as the coupled evolution of adaptive universe selection, gradient-based learning, and symbolic archive regularization. The analysis separates Generator- and archive-induced objective drift, establishes average SGD stationarity under cumulative drift, and proves fixed-comparator suppression under persistent signed utility advantages. Experiments across visual robustness, relational graph reasoning, and algorithmic graph reasoning exhibit stable optimization and archive concentration, although their effect on predictive performance is domain dependent. The present formulation is intentionally controlled: universes are finite and archive programs are simple deterministic rules. Extending PRAXIS to continuous or latent universe spaces, richer compositional or probabilistic archives, and a unified closed-loop theory of long-horizon co-evolution are natural directions for future work.

\bibliographystyle{neurips_2026}
\bibliography{neurips_2026}

\appendix
\section*{Appendix}

\section{Related Work}

\textbf{Self-Improving Learning Frameworks.}
A growing body of work studies self-improving and agentic systems in which models iteratively modify their data, policies, prompts, or intermediate reasoning artifacts through feedback-driven loops, including self-training frameworks \citep{huang2025r}, foundational agent architectures \citep{xia2025agent0,shang2024agentsquare}, and population-based evolutionary agents \citep{wang2025evoagentx}. These approaches demonstrate substantial empirical gains by leveraging archives of prior behaviors, reasoning traces, or candidate solutions, often guided by reward models, verifiers, or task success signals. However, such systems typically emphasize capability growth and system-level behavior, with limited formal analysis of the learning dynamics induced by repeated self-modification, particularly under non-stationary objectives. PRAXIS is complementary to this line of work: rather than proposing a new agent architecture or self-improvement mechanism, it provides a learning-theoretic framework for modeling and analyzing the coupled evolution of data selection, optimization, and symbolic structure that underlies many self-improving systems. In this sense, PRAXIS serves as an abstract dynamical model that captures common mechanisms in agentic and self-evolving frameworks, while enabling explicit guarantees on objective drift, stability, and long-horizon tracking.

\textbf{Adaptive Data Selection and Learning under Non-Stationary Objectives.}
ERM \cite{hanneke2024universal} constitutes the classical learning paradigm, optimizing average risk under a fixed and stationary data distribution. A large body of work extends ERM by adapting the training distribution over time, thereby implicitly inducing non-stationary learning objectives. Curriculum Learning (CL) methods reorder or reweight examples to accelerate convergence or improve representation learning, beginning with fixed difficulty schedules \citep{bengio2009curriculum} and extending to modern variants that emphasize training efficiency in large visual backbones \citep{wang2023efficienttrain} or integrate active selection criteria \citep{jafarpour2021active}. Distributionally Robust Optimization (DRO) addresses distribution shift by replacing the average-risk objective with a worst-case or group-wise objective, leading to minimax formulations whose solutions evolve as group weights are updated; recent work provides near-optimal algorithms and lower bounds for group DRO \citep{soma2022near}, multi-expert robust formulations \citep{jeongmulti}, and semi-supervised extensions that mitigate group sparsity using unlabeled data \citep{awasthi2024semi,jung2025per}. From a complementary perspective, bandit algorithms formalize adaptive data or policy selection as an online decision problem, balancing exploration and exploitation through regret-minimizing updates or, more recently, enabling valid statistical inference under adaptive data collection \citep{bibaut2021post,gutowski2021gorthaur}. Such methods have also been applied to adaptive components in large language model pipelines, including decoding and tool selection \citep{bouneffouf2025multi}. While these approaches differ in mechanism, they share a common structure: the learner is trained against an objective that changes over time as a consequence of adaptive selection, yet the resulting non-stationarity is typically treated as exogenous, adversarial, or heuristic rather than as a first-class dynamical quantity. In contrast, PRAXIS explicitly models data selection, parameter learning, and symbolic structure as a coupled closed-loop system, enabling theoretical guarantees that directly relate the rate and geometry of objective drift to stability and tracking behavior in non-stationary learning environments.

\textbf{Symbolic and Structured Inductive Bias in Learning.}
A broad class of neuro-symbolic methods incorporates symbolic structure into learning pipelines to improve compositionality, interpretability, or generalization. Differentiable logic frameworks such as DeepProbLog \cite{manhaeve2018deepproblog} and Scallop \cite{li2023scallop} integrate neural predicates with probabilistic or differentiable logic programs, enabling end-to-end training while predominantly assuming a fixed and trusted symbolic program provided as background knowledge. Complementary approaches address non-differentiable symbolic components through gradient estimation techniques, such as Inputâ€“Output Sample Gradient Estimation (ISED) \cite{solko2024data}, which enables data-efficient optimization of neural programs containing black-box symbolic modules or external APIs without access to internal structure. Other lines of work, including the Neuro-Symbolic Concept Learner (NSCL) \cite{mao2019neuro} and related program induction models, learn latent symbolic programs for reasoning over structured representations, demonstrating strong compositional generalization but focusing primarily on task-level accuracy rather than long-horizon learning dynamics. More recent programmable neuro-symbolic systems, such as Dolphin \cite{naik2024dolphin}, combine symbolic modules with adaptive task distributions to scale reasoning across problem families, while symbolic discovery methods extract interpretable rules or expressions from trained neural networks as post-hoc explanations. Across these approaches, symbolic structure is typically treated as fixed, correctness-assumed, or optimized solely for predictive performance, with limited consideration of how symbolic components should be regulated under non-stationary learning conditions. In contrast, PRAXIS treats symbolic programs as fallible inductive biases whose influence is dynamically modulated through interaction with adaptive data selection and parameter learning. By explicitly modeling the coupled evolution of symbolic structure, data distributions, and model parameters, PRAXIS enables theoretical analysis of stability, suppression, and tracking behavior in learning systems operating under endogenous non-stationarity.

\section{Symbolic Archive Composition}
\label{app:symbolic_archive}

The symbolic archive supplies a finite set of deterministic programs that
produce auxiliary predictions or structural preferences. It is not a separate
semantic oracle. Archive programs are evaluated through the same utility, blame,
and Gibbs-weight update mechanism across domains, so the reported suppression
behavior reflects PRAXIS dynamics rather than manual selection at evaluation
time.
When an archive program is used in the learner loss, its discrete output is
converted into a stop-gradient pseudo-label and compared with the learner's
soft prediction through the clipped discrepancy in
Eq.~\eqref{eq:archive_discrepancy}.

\begin{table}[h]
\centering
\scriptsize
\caption{Symbolic archive sizes and program families.}
\label{tab:archive_families}
\begin{tabular}{lll}
\toprule
Domain & Archive size & Program families \\
\midrule
MIT Indoor & $10$ initial programs & ranking, threshold, symmetry, temporal, distributional \\
GraphLog & $8$ programs & constant, prior, descriptor, uncertainty-gated \\
CLRS & $8$ programs & algorithmic, local graph rule, reachability, sparsity, weighted prior \\
\bottomrule
\end{tabular}
\end{table}

\paragraph{MIT Indoor archive.}
For MIT Indoor, each archive program maps the learner's class-probability vector
to a hard label. The initial archive contains $10$ manually seeded programs
from five families: ranking-based rules such as argmax, second-best, and top-2
margin switching; threshold-based rules based on confidence or entropy;
symmetry-breaking rules such as even/odd or low-index preferences; temporal and
modular rules conditioned on iteration or modulo patterns; and distributional
preferences such as center-biased label selection.

\paragraph{GraphLog archive.}
The GraphLog archive contains $8$ programs: \texttt{AlwaysYes},
\texttt{AlwaysNo}, \texttt{TargetPrior}, \texttt{WorldPrior},
\texttt{WorldTargetPrior}, \texttt{DescriptorLookup},
\texttt{LowConfidenceNo}, and \texttt{LowConfidenceTargetPrior}. These cover
constant predictors, relation and world priors, descriptor-based symbolic
lookup, and uncertainty-gated fallback rules. The argmax-style model program is
excluded in the main run so that archive concentration is not driven by copying
the learner's current prediction.

\paragraph{CLRS archive.}
The CLRS archive contains $8$ programs: \texttt{UnweightedBFS},
\texttt{DirectOrTwoHop}, \texttt{ChainDistancePrior}, \texttt{ReachablePrior},
\texttt{DirectEdgePrior}, \texttt{DegreeConnectivityPrior},
\texttt{SparseChainPrior}, and \texttt{WeightedSlackPrior}. These programs
encode lightweight algorithmic and graph-structural heuristics: unweighted
shortest-path search, local edge or two-hop rules, reachability and degree
priors, sparsity-sensitive chain priors, and weighted slack heuristics. Exact
oracle and argmax-copy programs are excluded from the main run.

\section{Implementation and Reproducibility Details}
\label{app:implementation_details}

All adaptive methods operate over the same train and held-out universe splits
within each domain. Uniform / ERM keeps a fixed uniform distribution over
universes. GroupDRO-style sampling upweights universes with lower recent
held-out accuracy. Bandit-style sampling treats universes as adaptive arms.
Generator-only keeps the KL-regularized PRAXIS Generator but disables archive
regularization. Full PRAXIS uses both adaptive universe selection and archive
regularization. The main multidomain diagnostics use $50$ PRAXIS iterations.

\subsection{Component Ablation Analysis}
\label{app:component_ablations}

This subsection evaluates the contribution of adaptive universe selection and symbolic archive adaptation.  We compare five roles documented in Table ~\ref{tab:praxis_ablations}: ERM with a fixed uniform mixture and no archive, adaptive generator training with no archive regularization, a frozen archive with adaptive generator, a frozen generator with adaptive archive, and full PRAXIS.  The goal is not to claim that every component dominates in every domain, but to identify when performance gains arise from the coupled system and when one mechanism carries most of the benefit.

\begin{table*}[t]
\centering
\scriptsize
\caption{Component ablations across domains.  Mean, Min, CVaR$_{10}$, and Std are computed across held-out universes.}
\label{tab:praxis_ablations}
\begin{tabular}{llrrrr}
\toprule
Domain & Method & Mean & Min & CVaR$_{10}$ & Std \\
\midrule
MIT Indoor & ERM fixed & 0.678 & \textbf{0.617} & 0.630 & \textbf{0.029} \\
MIT Indoor & No archive & 0.658 & 0.574 & 0.618 & 0.038 \\
MIT Indoor & Frozen archive & 0.672 & 0.590 & 0.624 & 0.034 \\
MIT Indoor & Frozen generator & 0.661 & 0.589 & 0.603 & 0.031 \\
MIT Indoor & Full PRAXIS & \textbf{0.694} & 0.602 & \textbf{0.639} & 0.031 \\
\midrule
GraphLog & ERM fixed & 0.718 & 0.429 & 0.429 & 0.282 \\
GraphLog & No archive & \textbf{0.912} & \textbf{0.857} & \textbf{0.857} & \textbf{0.053} \\
GraphLog & Frozen archive & 0.709 & 0.222 & 0.222 & 0.299 \\
GraphLog & Frozen generator & 0.746 & 0.429 & 0.429 & 0.258 \\
GraphLog & Full PRAXIS & 0.851 & 0.714 & 0.714 & 0.106 \\
\midrule
CLRS & ERM fixed & 0.718 & 0.583 & 0.596 & 0.063 \\
CLRS & No archive & 0.751 & 0.633 & 0.654 & 0.055 \\
CLRS & Frozen archive & 0.753 & \textbf{0.650} & 0.654 & 0.064 \\
CLRS & Frozen generator & 0.739 & 0.642 & 0.642 & \textbf{0.054} \\
CLRS & Full PRAXIS & \textbf{0.763} & \textbf{0.650} & \textbf{0.658} & 0.059 \\
\bottomrule
\end{tabular}
\end{table*}

The ablations reveal different component roles across domains.  On MIT Indoor, full PRAXIS obtains the strongest mean and CVaR$_{10}$, while freezing either adaptive component reduces tail robustness.  On CLRS, full PRAXIS also gives the best mean and CVaR$_{10}$, although the margin over the no-archive and frozen-archive variants is modest.  GraphLog behaves differently: adaptive universe selection without archive regularization is strongest, while frozen archive and frozen generator variants both underperform full PRAXIS.  This indicates that the GraphLog split is primarily solved by adaptive data selection, whereas the archive contributes more clearly to robustness in MIT Indoor and more weakly in CLRS.  These results support the modular interpretation of PRAXIS: the framework exposes the coupled dynamics, but the empirical value of each component depends on the domain and archive quality.

\FloatBarrier

\section{Theory Proofs}

\subsection{Proofs for Section \ref{subsec:generator-induced drift}}
\begin{proof}[Proof of Theorem~\ref{thm:kl_drift}]
Add and subtract $\mathbb E_{u\sim q_t}L_{\mathrm{aug},t+1}(\theta,u)$.
For every $\theta$,
\begin{align*}
|F_{t+1}(\theta)-F_t(\theta)|
&\le \left|\mathbb E_{q_{t+1}}L_{\mathrm{aug},t+1}
-\mathbb E_{q_t}L_{\mathrm{aug},t+1}\right|\\
&\quad+\left|\mathbb E_{q_t}
(L_{\mathrm{aug},t+1}-L_{\mathrm{aug},t})\right|\\
&\le M\|q_{t+1}-q_t\|_1+a_t.
\end{align*}
Taking the supremum over $\theta$ and applying Pinsker's inequality proves
\[
\delta_t\le M\|q_{t+1}-q_t\|_1+a_t
\le M\sqrt{2\mathrm{KL}(q_{t+1}\Vert q_t)}+a_t.
\]
When only the weights change, the archive contribution is a weighted sum of
terms in $[0,1]$, so the triangle inequality gives
$a_t\le\lambda\|w_{t+1}-w_t\|_1$. The convergence claim follows immediately
when both terms on the right tend to zero.
\end{proof}

\begin{proof}[Proof of Lemma~\ref{lem:Delta_le_delta}]
Fix any $t\ge 0$.
For every realization of the algorithmic randomness up to iteration $t+1$, we
have $\theta_{t+1}\in\Theta$.
Hence,
\[
\big|F_{t+1}(\theta_{t+1})-F_t(\theta_{t+1})\big|
\;\le\;
\sup_{\theta\in\Theta}\big|F_{t+1}(\theta)-F_t(\theta)\big|
\;=\;
\delta_t.
\]
Taking expectations over the randomness preserves the inequality, yielding
$\Delta_t\le \mathbb E[\delta_t]$.
\end{proof}

\begin{proof}[Proof of Corollary~\ref{cor:generator_drift}]
Theorem~\ref{thm:kl_drift} and the two exponential component bounds give
\[
\delta_t\le Mc e^{-\kappa(t-T_0)}+d e^{-\rho(t-T_0)}.
\]
Taking expectations and applying Lemma~\ref{lem:Delta_le_delta} proves the
along-iterate claim.
\end{proof}

\subsection{Proofs for Section \ref{subsec:archive_suppression_under_utility_advantage} }

\begin{proof}[Proof of Lemma~\ref{lem:endogenous_margin}]
Fix an arbitrary iteration $t$ and program index $k$.

\paragraph{Step 1: Bounding the blame term.}
By assumption, blame is learner-gated i.e., blame is incurred only on examples
that the learner classifies correctly at iteration $t$.
Equivalently, $B_{k,t}$ is an expectation of a bounded indicator (or bounded
loss) multiplied by the learner-correctness gate
$\mathbbm{1}\{\hat y_t(x)=y\}$.
In particular, for every $x$, the per-sample blame contribution is nonnegative
and at most $1$ whenever the gate is active, and it is $0$ when the gate is
inactive. Therefore, almost surely,
\[
0 \;\le\; \text{(per-sample blame contribution)} \;\le\; \mathbbm{1}\{\hat y_t(x)=y\}.
\]
Taking expectations over $(x,y)$ yields
\[
0 \;\le\; B_{k,t}
\;\le\; \mathbb{P}(\hat y_t(x)=y).
\]
By definition of the learner population error
$e_t := \mathbb{P}(\hat y_t(x)\neq y)$, we have
$\mathbb{P}(\hat y_t(x)=y)=1-e_t$, hence
\begin{equation}
0 \;\le\; B_{k,t} \;\le\; 1-e_t.
\label{eq:blame_bound}
\end{equation}

\paragraph{Step 2: Lower bounding utility gaps.}
Recall the population utility definition
$U_{j,t} := F_{j,t} - \gamma B_{j,t}$ with $\gamma\ge 0$.
For any two programs $i$ and $j$,
\begin{align}
U_{i,t}-U_{j,t}
&= (F_{i,t}-\gamma B_{i,t}) - (F_{j,t}-\gamma B_{j,t}) \nonumber\\
&= (F_{i,t}-F_{j,t}) - \gamma(B_{i,t}-B_{j,t}).
\label{eq:utility_gap_expand}
\end{align}
Using the blame bounds \eqref{eq:blame_bound} for each program,
we have $0\le B_{i,t}\le 1-e_t$ and $0\le B_{j,t}\le 1-e_t$, which implies
\[
B_{i,t}-B_{j,t} \;\le\; B_{i,t} \;\le\; 1-e_t.
\]
Substituting into \eqref{eq:utility_gap_expand} gives
\[
U_{i,t}-U_{j,t}
\;\ge\;
(F_{i,t}-F_{j,t}) - \gamma(1-e_t).
\]
Define
\[
\varepsilon^{\mathrm{endo}}_{i,j,t}
:=
(F_{i,t}-F_{j,t})-\gamma(1-e_t).
\]
Then the preceding inequality is exactly
$U_{i,t}-U_{j,t}\ge \varepsilon^{\mathrm{endo}}_{i,j,t}$, as claimed. Hence, proved.

\end{proof}

\begin{proof}[Proof of Theorem~\ref{thm:archive_suppression_subg}]
Fix $j\notin\mathcal S$ and write $i=i(j)$. The Gibbs weights satisfy the exact
pairwise identity
\[
\frac{w_{j,t}}{w_{i,t}}
=\frac{w_{j,T_0}}{w_{i,T_0}}
\exp\!\left(\beta\eta_p\sum_{\tau=T_0}^{t-1}
(\widehat U_{j,\tau}-\widehat U_{i,\tau})\right).
\]
Since $\sum_{k\in\mathcal S}w_{k,t}\ge w_{i,t}$ and
$U_{j,\tau}-U_{i,\tau}=-\varepsilon_{j,\tau}$,
\[
\frac{w_{j,t}}{\sum_{k\in\mathcal S}w_{k,t}}
\le \frac{w_{j,T_0}}{w_{i,T_0}}
\exp\!\left(-\beta\eta_p\sum_{\tau=T_0}^{t-1}\varepsilon_{j,\tau}
+\beta\eta_p\sum_{\tau=T_0}^{t-1}(\xi_{j,\tau}-\xi_{i,\tau})\right).
\]
For a fixed ordered pair, decompose the noise difference into its conditional
mean and centered martingale difference. The conditional means sum to at most
$2\mu(t-T_0)$. The conditional $(2\sigma)$-sub-Gaussian MGF assumption and a
Chernoff bound give, for $t>T_0$,
\[
\mathbb P\!\left(\sum_{\tau=T_0}^{t-1}d^{(i,j)}_\tau\ge a\right)
\le \exp\!\left(-\frac{a^2}{8\sigma^2(t-T_0)}\right).
\]
A union bound over the at most $m^2$ ordered pairs and
$T-T_0$ positive-length horizons, with the zero-length case handled
deterministically, yields the simultaneous bound $\Gamma_{\mathrm{subG}}$.
Substitution in the preceding display proves the result.
\end{proof}

\begin{proof}[Proof of Corollary~\ref{cor:archive_exponential_suppression}]
Apply Theorem~\ref{thm:archive_suppression_subg} and take logarithms. Split the
cumulative gap at $T_1$ and use the assumed signed-margin bound. For
$i=i(j)$ this gives
\[
\log\frac{w_{j,t}}{\sum_{k\in\mathcal S}w_{k,t}}
\le C_j-\beta\eta_p\varepsilon(t-T_1)
+\beta\eta_p\Gamma_{\mathrm{subG}}(t,\delta),
\]
where
\[
C_j:=\log\frac{w_{j,T_0}}{w_{i,T_0}}
-\beta\eta_p\sum_{\tau=T_0}^{T_1-1}\varepsilon_{j,\tau}
\]
depends only on the score state and finite history through $T_1$. If $\mu=0$,
$\Gamma_{\mathrm{subG}}(t,\delta)=O(\sqrt{t\log T})$ for fixed horizon $T$,
whereas the signed signal is linear in $t$.
\end{proof}

\subsection{Proof for Section \ref{sec:learner_tracking}}
\begin{proof}[Proof of Theorem~\ref{thm:sgd_tracking_drift}]
By $L$--smoothness and $\theta_{t+1}=\theta_t-\eta g_t$,
\[
F_t(\theta_{t+1})\le F_t(\theta_t)
-\eta\langle\nabla F_t(\theta_t),g_t\rangle
+\frac{L\eta^2}{2}\|g_t\|^2.
\]
Conditioning on $\mathcal G_t$, unbiasedness and the variance bound imply
\[
\mathbb E[F_t(\theta_{t+1})\mid\mathcal G_t]
\le F_t(\theta_t)-\eta\left(1-\frac{L\eta}{2}\right)
\|\nabla F_t(\theta_t)\|^2+\frac{L\eta^2}{2}\sigma^2.
\]
For $\eta\le1/L$, the gradient coefficient is at least $\eta/2$. Pathwise,
$F_{t+1}(\theta_{t+1})\le F_t(\theta_{t+1})+\delta_t$; hence, after taking
total expectation,
\[
\mathbb E F_{t+1}(\theta_{t+1})
\le \mathbb E F_t(\theta_t)-\frac{\eta}{2}
\mathbb E\|\nabla F_t(\theta_t)\|^2
+\frac{L\eta^2}{2}\sigma^2+\mathbb E\delta_t.
\]
Summing from $t=0$ to $T-1$, telescoping, and using
$F_T(\theta_T)\ge F_{\min}$ gives
\[
\frac{\eta}{2}\sum_{t=0}^{T-1}\mathbb E\|\nabla F_t(\theta_t)\|^2
\le F_0(\theta_0)-F_{\min}+\frac{L\eta^2T}{2}\sigma^2
+\sum_{t=0}^{T-1}\mathbb E\delta_t.
\]
Division by $\eta T/2$ proves the claim.
\end{proof}

\end{document}